\documentclass[smallextended]{svjour3}
\usepackage[utf8]{inputenc}
\usepackage[english]{babel}

\usepackage{amssymb}
\usepackage{amsfonts}
\usepackage{bm}
\usepackage{mathtools}
\usepackage{color,soul}
\usepackage{mdframed}

\usepackage{wrapfig}
\usepackage{stfloats}

\usepackage[most,breakable]{tcolorbox}
\usepackage[numbers]{natbib}
\usepackage[colorlinks=true]{hyperref}

\DeclarePairedDelimiter\floor{\lfloor}{\rfloor}

\renewcommand\hl[1]{#1} %uncomment to disable highlighting

\begin{document}

\title{STRATA: Unified Framework for Task Assignments in Large Teams of Heterogeneous Agents}

\author{Harish Ravichandar \and
        Kenneth Shaw \and Sonia Chernova
        }

\institute{Georgia Institute of Technology, Atlanta, GA. \\
           Email: {\tt\small \{harish.ravichandar, kshaw, chernova\} @gatech. edu}
           }%

\maketitle

\begin{abstract}
    Large teams of heterogeneous agents have the potential to solve complex multi-task problems that are intractable for a single agent working independently. However, solving complex multi-task problems requires leveraging the relative strengths of the different kinds of agents in the team.  We present \emph{Stochastic TRAit-based Task Assignment (STRATA)}, a unified framework that models large teams of heterogeneous agents and performs effective task assignments. Specifically, given information on which \emph{traits} (capabilities) are required for various tasks, STRATA computes the assignments of agents to tasks such that the trait requirements are achieved. Inspired by prior work in robot swarms and biodiversity, we categorize agents into different \emph{species} (groups) based on their traits. We model each trait as a \emph{continuous} variable and differentiate between traits that can and cannot be aggregated from different agents. STRATA is capable of reasoning about both species-level and agent-level variability in traits. Further, we define measures of \emph{diversity} for any given team based on the team's continuous-space trait model. We illustrate the necessity and effectiveness of STRATA using detailed experiments based in simulation and in a capture-the-flag game environment.
\end{abstract}

\section{Introduction}

The study of multi-agent systems has produced significant insights into the process of engineering collaborative behavior in groups of agents \cite{olfati2007consensus, bordini2007programming}. These insights have resulted in large teams of agents capable of accomplishing complex tasks that are intractable for a single agent, with applications including environmental monitoring \cite{shkurti2012multi}, agriculture \cite{tokekar2016sensor}, warehouse automation \cite{wurman2008coordinating}, construction \cite{werfel2014designing}, defense \cite{beni2004swarm}, and targeted drug delivery \cite{li2017micro}. Efficient solutions to the above problems typically rely on a wide range of capabilities. %Further, capability requirements often dynamically change owing to the harsh operating conditions.
Teams of heterogeneous agents are particularly well suited for performing complex tasks that require a variety of skills, since they can leverage the relative advantages of the different agents and their capabilities.
In this work, we are motivated by robotics applications, and the multi-robot task assignment (MRTA) problem in particular \cite{gerkey2004formal, korsah2013comprehensive, khamis2015multi} which formally defines the challenges involved in optimally assigning agents to tasks.
%The multi-robot task assignment (MRTA) problem \cite{gerkey2004formal, korsah2013comprehensive, khamis2015multi} formally define the challenges involved in optimally assigning agents to tasks.

We present \emph{Stochastic TRAit-based Task Assignment (STRATA)}, a unified modeling and task assignment framework, to solve an instance of the MRTA problem with an emphasis on large heterogeneous teams.
%STRATA enables a heterogeneous team of agents to effectively divide the various tasks among its members.
We model the topology of tasks as a strongly connected graph, with each node representing a task and or a physical location and the edges indicating the possibility of switching between any two tasks. We assume that the optimal agent-to-task associations are unknown and that the task requirements are specified in terms of the various \emph{traits} (capabilities) required for each task. Thus, in order to effectively perform the tasks, the agents must reason about their combined capabilities and the limited resources of the team. To enable this reasoning, we take inspiration from prior work in robot swarms \cite{prorok2017impact} and biodiversity \cite{petchey2002functional}, and propose a group modeling approach \cite{albrecht2018autonomous} to model the capabilities of the team. Specifically, we assume that each agent in the team belongs to a particular \emph{species}~\footnote{\hl{Similar to prior work in multi-robot systems} \cite{prorok2017impact}, \hl{we use the term ``species" to describe a group of agents with similar traits.  This does not imply any similarity to biological species.}}. Further, each species is defined based on the traits possessed by its members. Assuming that the agents are initially sub-optimally assigned to tasks on the task graph, STRATA computes assignments such that the agents can reorganize themselves to collectively aggregate the traits necessary to meet the task requirements as quickly as possible.

\begin{figure*}[t]
    \centering
    \includegraphics[scale=0.45] {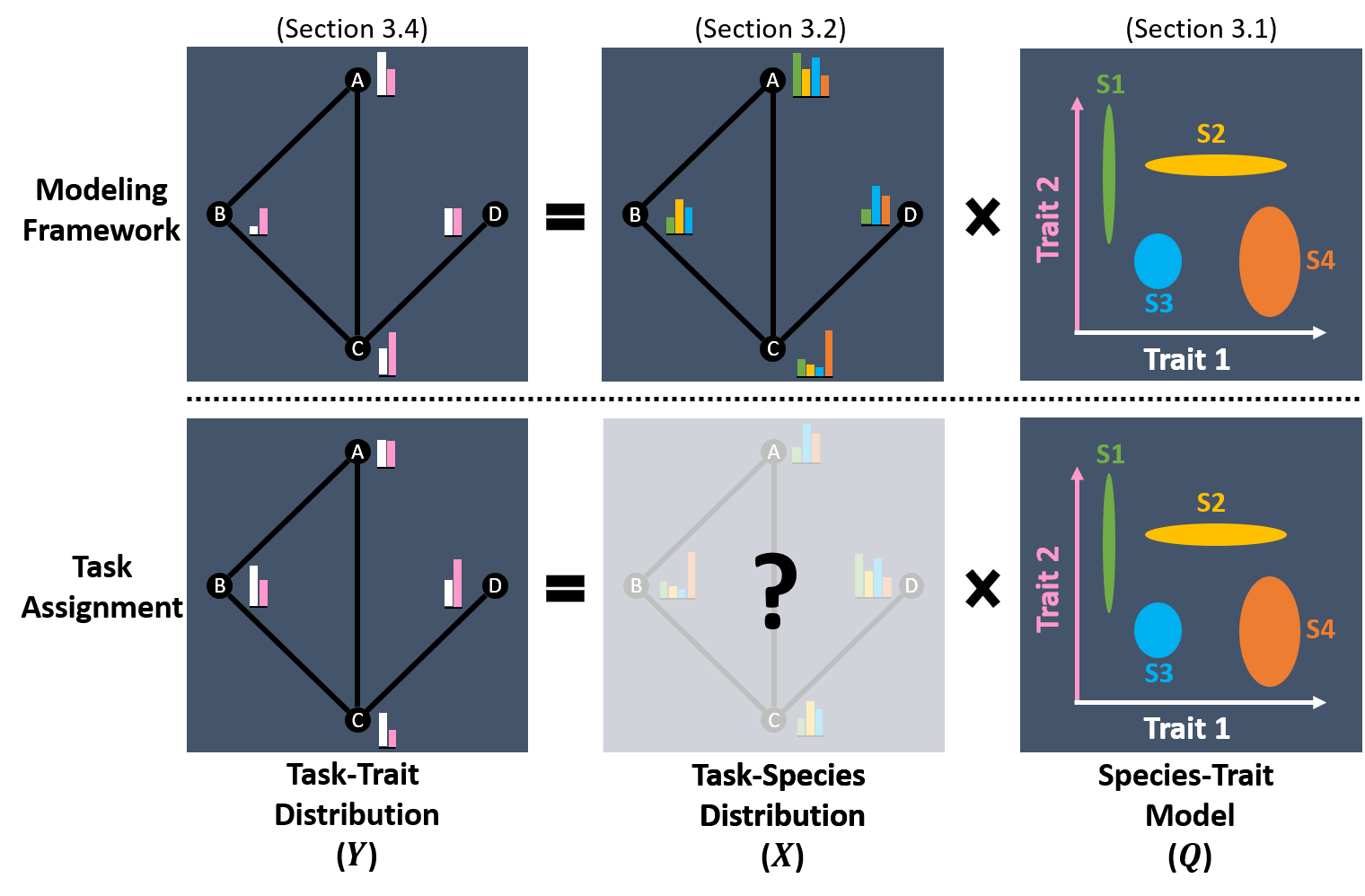}
    \caption{\textit{Top row:} STRATA defines the effects of task-species distribution and the species-trait model on the task-trait distribution. \textit{Bottom row}: Given a team defined by the species-trait model, we aim to perform task assignments such that the desired task-trait distribution is achieved.}
    \label{fig:modeling_framework}
\end{figure*}

\hl{
In Fig. {\ref{fig:modeling_framework}}, we illustrate the basic building blocks of STRATA and the task assignment problem. As seen in the top row, STRATA models the effects of task assignments (task-species distribution $X$) and the species' traits (species-trait model $Q$) on how the traits are aggregated for each task (task-trait distribution $Y$). The example species-trait model in Fig. {\ref{fig:modeling_framework}} illustrates the distribution of two traits (Trait 1 and Trait 2) in four different species ($S1$, $S2$, $S3$, and $S4$). Colored bar graphs near each node of the graph are used to illustrate either the task-trait ($Y$) or the task-species ($X$) distributions. As can be seen, for any distribution of agents across the tasks, we can compute the corresponding trait distribution across the tasks. We also derive a closed-form expression to quantify the effect of the variability in the agents' traits on the achieved task-trait distribution. The task assignment problem, as seen in the bottom row, involves computing the optimized task assignments, given a desired distribution of traits across the tasks and a team of heterogeneous agents described by a species-trait distribution.
}

Using the above model, STRATA allows for the optimization of two separate task assignment goals: (1) \emph{exact matching} and (2) \emph{minimum matching}. In exact matching, the algorithm aims to distribute the agents such that the achieved trait distribution is as close to the desired as possible. In minimum matching, the algorithm aims to distribute the agents such that the achieved trait distribution is higher than or equal to the desired as possible, i.e., over-provisioning is not penalized.

The STRATA representation of both task and species traits is inspired by \cite{prorok2017impact}, which considered binary instantiations of traits. However, binary models fail to capture the nuances in the scales and natural variations of the agents' traits. For instance, consider an unmanned aerial vehicle and a bipedal robot. While both agents share the mobility trait (the ability to move), their speeds are likely to be considerably different. To address these challenges, in STRATA we have extended the representation to model traits in the \emph{continuous} space. Additionally, STRATA also captures agent-level differences within each species by using a \emph{stochastic} trait model.

When reasoning about the collective strengths of the team, attention must be paid to the fact that not all capabilities are improved in quantity by aggregation of individual agents' abilities. For instance, a coalition of any number of slow robots does not compensate for a faster robot. Taking this observation into account, we consider two types of traits: \emph{cumulative} and \emph{non-cumulative}. We consider a trait to be (non) cumulative if it can (not) be aggregated from different agents in order to achieve certain task requirements.

In this work, we also extend the \emph{diversity measures} introduced in \cite{prorok2017impact} to the continuous space. We derive two separate diversity measures, one for each goal function. The diversity measures provide insights about the trait-based heterogeneity of the team. Specifically, the diversity measures help define a a minimum subset of the species that can collectively compensate for the rest of the team.

In summary, the key contributions of our work include a unified framework for effective task assignment of large heterogeneous teams that:
\begin{enumerate}
  \item incorporates a stochastic trait model that captures both between-species and within-species variations,
  \item assigns tasks to agents with respect to two separate goals: exact matching and minimum matching, and
  \item computes measures of diversity in teams with continuous trait models.
\end{enumerate}

We evaluate STRATA using detailed simulations and a capture-the-flag game environment. Our results demonstrate the necessity and effectiveness of STRATA in terms of effective task assignment and improved team performance, when compared to a baseline that only considers binary traits.

Finally, we note that heterogeneous teams can be composed of both robotic and human agents.  Human traits can be vastly different from and complementary to those of robots \cite{decostanza2018enhancing}. For instance, when compared to humans, robots can carry heavier payloads, move faster, and be immune to fatigue. On the other hand, humans' abilities to assimilate and maintain situational awareness, process noisy information, and adapt to highly unstructured environments are unmatched by the abilities of their robotic counterparts. Further, individual differences are considerable in teams involving humans (see \cite{decostanza2018enhancing} and references therein). Although not yet evaluated with human-robot teams, STRATA's ability to characterize humans as one or more separate species (e.g., soldiers, pilots, medics) possessing stochastic traits makes it a promising representation for modeling human-robot teaming.

\section{Related Work}

Significant efforts have been focused on problems in multi-robot systems task assignment (MRTA) \cite{gerkey2004formal, korsah2013comprehensive, khamis2015multi}. Broadly, the problems are categorized based on three binary characteristics: (1) Task type (single-robot [SR] vs multi-robot [MR]), (2) robot type (single-task [ST] vs multi-task [MT], and (3) assignment type (instantaneous [IA] vs time-extended [TA]) \cite{gerkey2004formal}. While task type indicates the number of robots required to complete each task, robot type indicates whether the robots are capable of simultaneously performing a single task or multiple tasks. The assignment type is used to differentiate between tasks that involve scheduling constraints and those that do not. Indeed, numerous approaches related to the different variants of MRTA, including assignments involving single-robot tasks, are available in the literature. However, we limit our coverage of related work to variants of MRTA that involve multi-robot tasks - tasks that involve the coordination of several robots. We refer readers to \cite{gerkey2004formal,korsah2013comprehensive} for comprehensive categorizations and examples of all approaches pertaining to task assignment.

Several methods for task assignment with homogeneous agents have been proposed. A graph-theoretic framework, named SCRAM, is proposed in \cite{AAAI15-MacAlpine}. SCRAM maps agents to target locations while simultaneously avoiding collisions and minimizing the time required to reach target locations. The work in \cite{ma2016optimal} presents a hierarchical algorithm that is correct, complete, and optimal for simultaneously task assignment and path finding. A fast bounded-suboptimal algorithm, that automatically generates highways for a team of homogeneous agents to reach their target locations, is introduced in \cite{cohen2016improved}. Notably, the methods in \cite{AAAI15-MacAlpine, ma2016optimal,cohen2016improved} emphasize optimal path finding for each robot and collision avoidance in order to assign each robot to a single task (reaching a target location). However, these methods assume that all the agents in the team are interchangeable, and thus are not suitable for multi-task scenarios that involve several heterogeneous agents.

Approaches involving single-task robots solve the assignment problem by assuming that each robot is specialized and can only perform one task. The method proposed in \cite{shehory1995kernel} addresses a transportation task involving multiple single-task robots. Some of the items to be transported can be transported by a single robot, while others need coordinated efforts from several robots. \cite{shehory1995kernel} use a greedy set-partitioning algorithm to form coalitions of robots required to perform the tasks. Potential coalitions are iteratively computed for each task involved. The coalition formation algorithm introduced in \cite{shehory1995kernel} was later extended in \cite{vig2006multi}. The extended algorithm in \cite{vig2006multi} reduces the communication effort, balances the coalitions, and constrains the requirements to specify if and when all the required traits must be possessed by a single robot. These approaches, however, require the listing of all potential coalitions and thus are not suitable for problems involving large number of possible coalitions. Indeed, the number of possible coalitions is a factor of both the number of robots in the team and the inherent diversity of the team. Specifically, as the number of robots in the team and their similarities increase, so does the number of possible coalitions. STRATA, on the other hand, is scalable with the number of agents as it does not list all possible coalitions.

Decentralized approaches for task assignment are introduced in \cite{jang2018anonymous, berman2009optimized, hsieh2008biologically, matthey2009stochastic}. A game-theoretic task assignment strategy is introduced in \cite{jang2018anonymous} to assign tasks to a team of homogeneous robots with social inhibition. In \cite{berman2009optimized}, multiple tasks are assigned to a team of homogeneous robots. The authors develop of a continuous abstraction of the team by modeling population fractions and defining the task allocation problem as the selection of rates of robot switching from and to each task. In \cite{hsieh2008biologically}, the authors extend the method in \cite{berman2009optimized} with a wireless communication-free quorum sensing mechanism in order to reduce task assignment time. In \cite{matthey2009stochastic}, a decentralized approach for heterogeneous robot swarms is introduced. The approach computes optimal rates at which the robots must switch between the different tasks. These rates, in turn, are used to compute probabilities that determine stochastic control policies of each robot. However, a common shortcoming of these decentralized approaches is that they assume that the desired behavior is specified as a function of the distribution of agents across the tasks.

Auction or market-based methods also provide solutions to the MRTA problem involving single-task robots \cite{guerrero2003multi,lin2005combinatorial, vail2003multi, dias2006market}. In \cite{guerrero2003multi}, the robot responsible for any given task is the robot who discovers the task. Once discovered, the robot holds an auction to recruit other robots into a coalition. \cite{lin2005combinatorial} introduce combinatorial bidding to form coalitions and provides explicit cooperation mechanism for robots to form coalitions and bid for tasks. A homogeneous task assignment algorithm for robot soccer is presented in \cite{vail2003multi}. Sensed information from the robots are shared to compute a shared potential function that would help the robots move in a coordinated manner. We refer readers to \cite{dias2006market} for a survey of market-based approaches applied to multi-robot coordination. A common trait of auction or market-based methods is that they require extensive communication for bidding and scale poorly with the number of robots in the team. Further, the methods discussed thus far are limited to either single-robot tasks or single-task robots. In contrast, STRATA considers agents capable of performing tasks that require coordination between multiple agents.

Our work falls under the category of \emph{Single-Task Robots Multi-Robot Tasks Instantaneous Assignment (ST-MR-IA)} problem, also known as the \emph{coalition formation} problem \cite{gerkey2004formal}. In other words, we are interested in assigning a team of agents to several tasks, each requiring several agents. The assignment type is instantaneous since our task assignment does not reason about future task assignments or scheduling constraints. The ST-MR-IA is an instance of the set-partitioning problem in combinatorial optimization and is known to be strongly NP-Hard \cite{gerkey2004formal}. Albeit not developed for MRTA, a few efficient approximate solutions have been proposed for the set partitioning problem \cite{atamturk1996combined, hoffman1993solving}. Based on prior work in set partitioning problems, centralized and distributed algorithms to solve a ST-MR-IA problem have been proposed in \cite{shehory1995kernel, shehory1998methods}. The method in \cite{shehory1998methods} has also been adapted to be more efficient by reducing the required communication and discouraging imbalanced coalitions \cite{vig2006multi}. A method for coalition formation is introduced in \cite{parker2006building} by building a solution to a task by dynamically connecting a network of behaviors associated with individual robots.

A limitation of most of the existing approaches is that the desired behavior is assumed to be specified in terms of optimal agent distribution. A notable exception to this generalization is the framework introduced in \cite{prorok2017impact}, which is capable of receiving the task requirements provided in the form a desired trait distribution cross tasks. We take a position similar to \cite{prorok2017impact}, and do not assume that the desired distribution of agents is known. Another similarity between STRATA and \cite{prorok2017impact} is being suitable for a decentralized implementation. Thus, both approaches are scalable in the number of agents and their capabilities, and are robust to changes in the agent population.

While STRATA shares several similarities with \cite{prorok2017impact}, there are a number of notable relative advantages. First, our species-trait model is continuous, while \cite{prorok2017impact} uses a binary model. Second, we differentiate between cumulative and non-cumulative traits. Third, the framework in \cite{prorok2017impact} utilizes a deterministic model of traits. In contrast, we consider the inherent variability in the agent' traits, thereby capturing the variations at both species and agent levels. Finally, while the diversity measures introduced in \cite{prorok2017impact} are limited to binary trait models, our measures are compatible with continuous-space models.

\section{Modeling framework}\label{sec:model}
In this section, we introduce the various elements of STRATA that enables task assignments in large heterogeneous teams.  Assigning tasks to the different agents in the team requires reasoning about their complementary traits and the limited resources of the team. STRATA handles this challenge using (1) a \emph{stochastic trait model} that describes the capabilities of each species in the team along with the corresponding variance, (2) an \emph{task graph} that describes the dynamics and constraints associated with agents traversing the task graph, and (3) a \emph{agent distribution model} that  specifies how agents are distributed across the various tasks. Together, these models explain the combinations of capabilities that are currently available at each task.

Based on the above mentioned models, we formulate and solve a constrained optimization problem to distribute the agents across the different tasks to satisfy certain trait-based task requirements. Specifically, we compute the transition rates on the task graph, which in turn dictate how task assignments vary as a function of time such that the desired trait distribution is achieved and maintained as quickly as possible. Further, our optimization explicitly reasons about the expected variance of the trait distribution.

Throughout the paper, we illustrate STRATA using a running example of a task assignment problem. We will progressively build the example as we introduce the different parts of the framework.

\subsection{Trait Model}
\label{subsec:trait_model}
\textbf{Base model:} Consider a heterogeneous team of agents. We assume that each agent is a member of a particular species. The number of species $S \in \mathbb{N}$ is finite, and the number of agents in the $s^\mathrm{th}$ species is denoted by $N_s$. Thus, the total number of agents in the team is given by $N = \sum_{s=1}^{S} N_s $ We define each species by its abilities (\emph{traits}). Specifically, the traits of each species are defined as follows
\begin{equation}
    q^{(s)} \triangleq [q^{(s)}_1, q^{(s)}_2, \cdots ,q^{(s)}_U], \quad \forall s = 1, 2, \cdots, S
    \label{eq:q_s}
\end{equation}
where $q^{(s)}_u \in \mathbb{R}_{+}$ is the $u^\mathrm{th}$ trait of the $s^\mathrm{th}$ species, and $U$ is the number of traits. Thus, the traits of the team is defined by a $S \times U$ \emph{species-trait matrix} $Q = [q^{(1)^T}, \cdots, q^{(S)^T}]^T \in \mathbb{R}^{S \times U}_{+}$, with each row corresponding to one species and each column corresponding to one trait.

\vspace{10pt}
\textbf{Stochastic traits}: To capture the natural variability in the capabilities of each species, we maintain a stochastic summary of each species' traits. Specifically, each element of $Q$ is assumed to be an independent Gaussian random variable, $q^{(s)}_u \sim \mathcal{N}(\mu_{su}, \sigma^2_{su})$. Thus, a vector random variable with all the traits of all the species can be written as $q = [q^{(1)},q^{(2)},\cdots,q^{(S)}] \sim \mathcal{N}(\mu_q, \Sigma_q)$, with its mean given by
\begin{equation}
    \mu_q = [\mu_{q^{(1)}}, \cdots, \mu_{q^{(S)}}] \in \mathbb{R}^{SU}_{+}
    \label{eq:mean_q}
\end{equation}
where $\mu_{q^{(s)}} = [\mu_{s1}, \cdots , \mu_{sU}] \in \mathbb{R}^{U}$ contains the expected trait values of the $s^{\mathrm{th}}$ species, and its covariance given by the following block-diagonal matrix
\begin{equation}
    \Sigma_q = \mathrm{diag}([\Sigma_{q^{(1)}}, \cdots, \Sigma_{q^{(S)}}]) \in \mathbb{R}^{SU \times SU}_{+}
    \label{eq:Sigma_q}
\end{equation}
where $\Sigma_{q^{(s)}} = \mathrm{diag}([\sigma_{s1}^2, \cdots , \sigma_{sU}^2]) \in \mathbb{R}^{U \times U}_{+}$ is the diagonal covariance matrix associated with the $s^{\mathrm{th}}$ species. The expected trait values can be rewritten in the form of the \emph{expected species-trait matrix} $\mu_Q = [\mu_{q^{(1)}}^T, \cdots, \mu_{q^{(S)}}^T]^T \in \mathbb{R}^{S \times U}_{+}$ as follows
\begin{align}
    &\bm{\mu}_Q =
    \begin{bmatrix}
    \mu_{11} & \cdots & \mu_{1U} \\
    \vdots & \ddots & \vdots \\
    \mu_{S1} & \cdots & \mu_{SU}
    \end{bmatrix}
    \label{eq:mu_Q}
\end{align}
Similarly, the non-zero diagonal elements of the covariance matrix can be rewritten in matrix form as
\begin{align}
    \mathrm{Var}_Q =
    \begin{bmatrix}
    \sigma_{11}^2 & \cdots & \sigma_{1U}^2 \\
    \vdots & \ddots & \vdots \\
    \sigma_{S1}^2 & \cdots & \sigma_{SU}^2
    \end{bmatrix}
    \label{eq:Var_Q}
\end{align}

While the above stochastic model can be directly specified, it can also be learned directly from data. Given the trait values of each agent in the team, clustering algorithms, such as the Gaussian mixture model (GMM), can be used to automatically find Gaussian clusters, each of which can be interpreted as a probabilistic representation of a species in the trait space.

\begin{tcolorbox}[colback=blue!5!white, enhanced jigsaw, breakable]
\textit{Example}: Consider an example scenario in which the team is made up of four species, each consisting of $25$ agents. Each species is characterized by the following four traits: \emph{i)} coverage area ($m^2$): a function of sensing capabilities, \emph{ii)} area of influence ($m^2$): a function of actuation capabilities, \emph{iii)} number of health packs: a function of payload capacity, and \emph{iv)} ammunition: another function of payload capacity. Let the expected value of the species-trait matrix and the corresponding matrix of variances for our example team be given by

\begin{equation}
    \bm{\mu}_Q =
    \begin{bmatrix}
    50 & 15 & 20 & 140 \\
    150 & 10 & 10 & 0 \\
    175 & 0 & 25 & 60 \\
    200 & 35 & 30 & 140
    \end{bmatrix}, \quad
    \mathrm{Var}_Q =
    \begin{bmatrix}
    3 & 1 & 1.5 & 5.6 \\
    2 & 1.5 & 0.5 & 0 \\
    1 & 0 & 2.4 & 8.7 \\
    6 & 2.3 & 3.9 & 9.2
    \end{bmatrix}
    \label{eq:ex_mu_Q_and_var_Q}
\end{equation}

Note that STRATA allows for modeling traits of different orders of magnitude. Further, for the same trait, the variation observed in each species is different. For instance, consider the ammunition trait ($4^\mathrm{th}$ columns of $\mu_Q$ and $\mathrm{Var}_Q$). The distribution of this trait is considerably different in each species. Specifically, while Species 1 has the largest average units of ammunition ($140$), it also has the smallest variance ($5.6$). On the other hand, Species 3 has considerably lower units of ammunition ($60$) while its variance ($8.7$) is considerably higher than that of Species 1. Encoding these aspects of each species enables STRATA to reason about the various trade-offs involved in recruiting agents to meet the task requirements.
\end{tcolorbox}

\subsection{Task Graph}
Given the trait model from the previous section, we require the team to accomplish $M$ tasks. We model the topology of the tasks using a strongly connected graph $\mathbb{G} = (\mathcal{E},\mathcal{V})$. The vertices $\mathcal{V}$ represent the $M$ tasks, and the edges $\mathcal{E}$ connect tasks such that the existence of an edge between two tasks represents the agents' ability to switch between them. For each species, we aim to optimize the transition rate $k_{ij}^{(s)}$ for every edge in $\mathcal{E}$, such that $0<k_{ij}^{(s)}<k_{ij,max}^{(s)}$.

The transition rate $k_{ij}^{(s)}$ defines the rate at which an agent from species $s$ currently performing task $i$ switches to task $j$. The limits on transition rates can help incorporate realistic constraints, such as the time required to travel between physical locations and to change the tools required to perform tasks. The transition rates implicitly dictate how the distribution of agents across tasks evolves in time. Note that, to account for the differences among species, the transition rates and their limits are defined separately for each species.

\begin{tcolorbox}[colback=blue!5!white, enhanced jigsaw, breakable]
\textit{Example}: Let our example problem involve 5 tasks and the task graph is shown in Fig. \ref{fig:example_graph}. Note that the graph is not fully connected. This reflects the restrictions on how the agents can switch between tasks. For instance, let each task be carried out in a different physical location. The presence (absence) of an edge between any two tasks implies that it is (not) possible for the agents to move between the two tasks. STRATA explicitly takes these restrictions into consideration when distributing agents across the task graph.
\end{tcolorbox}

\begin{figure}[htb]
    \centering
    \includegraphics[scale=0.13] {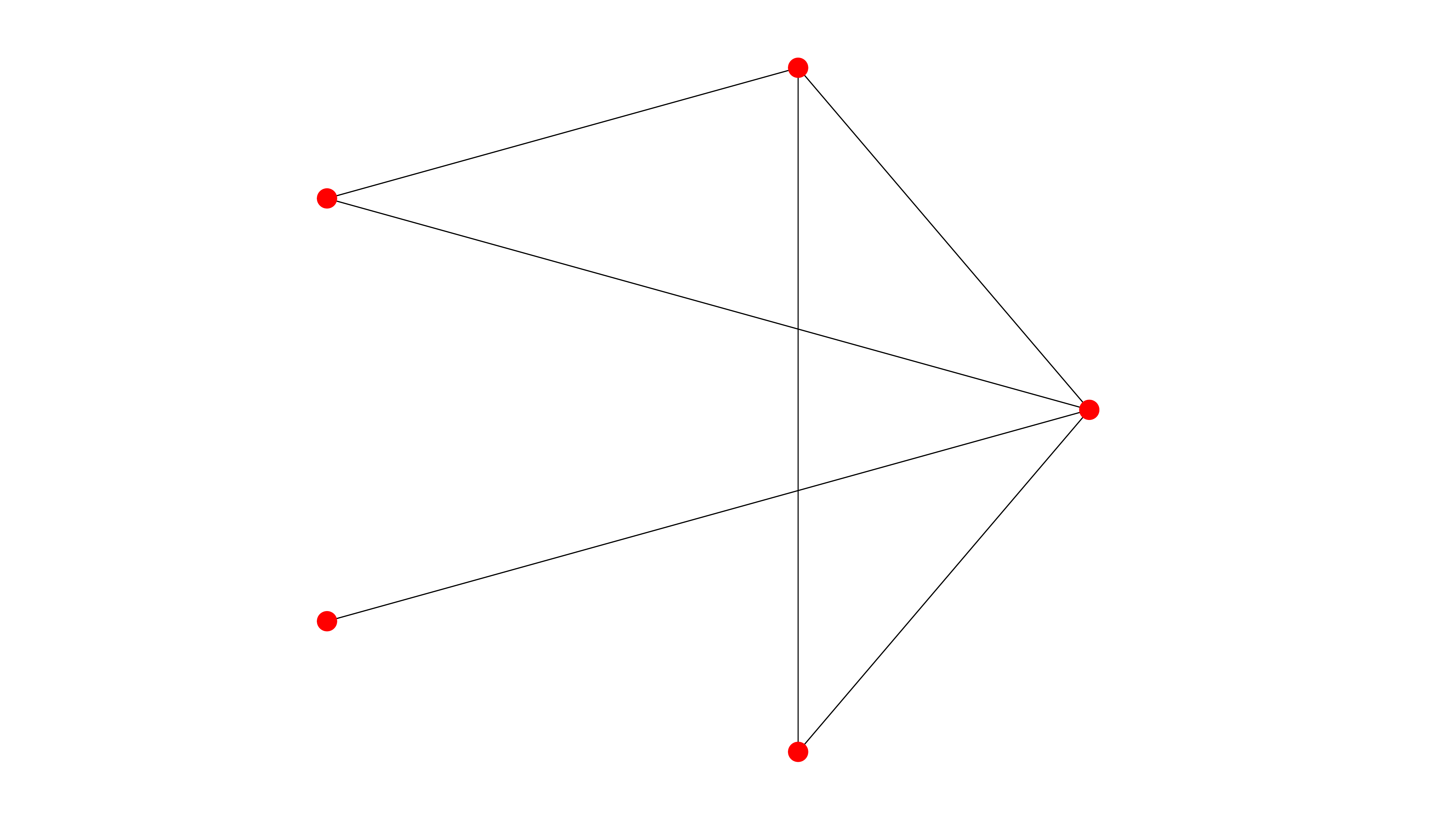}
    \caption{The task graph associated with our example task assignment problem.}
    \label{fig:example_graph}
\end{figure}

\subsection{Agent Distribution}
With the capabilities of the team and the tasks defined, the modeling of individual agents and their assignments remains. The distribution of agents from species $s$ across the $M$ tasks at time $t$ is defined by $\mathrm{x}^{(s)}(t) =$ $[x^{(s)}_{1}(t), x^{(s)}_{2}(t),$ $\cdots x^{(s)}_{M}(t)]^T$ $\in \mathbb{N}^M$. Thus the distribution of the whole team across the tasks at time $t$ can be described using a \emph{abstract state information matrix} $\bm{X}(t) = [\mathrm{x}^{(1)}(t), \mathrm{x}^{(2)}(t), \cdots, \mathrm{x}^{(S)}(t)] \in \mathbb{N}^{M \times S}$.

\begin{tcolorbox}[colback=blue!5!white, enhanced jigsaw, breakable]
\textit{Example}: Let us assume that the initial distribution of agents, perhaps a result of initial deployment or earlier task requirements, is given by
\begin{equation}
    \bm{X}(0) =
    \begin{bmatrix}
    25 & 0 & 0 & 0 \\
    0 & 25 & 0 & 0 \\
    0 & 0 & 25 & 0 \\
    0 & 0 & 0 & 25 \\
    0 & 0 & 0 & 0
    \end{bmatrix}. \label{eq:ex_X0}
\end{equation}
Thus, initially, all the $25$ agents from Species 1 are assigned to Task 1, all agents from Species 2 to Task 2, and so on. Further, no agents are assigned to Task 5. Note that each column adds up to the number of agents in the corresponding species.
\end{tcolorbox}

As in \cite{prorok2017impact}, the time evolution of the number of agents from Species $s$ at Task $i$ is explained by the following dynamical system
\begin{equation}
    \dot{x}^{(s)}_{i}(t) = \sum_{\forall j|(i,j) \in \mathcal{E}} k_{ji}^{(s)} x_j^{(s)}(t) - k_{ij}^{(s)} x_i^{(s)}(t)
\end{equation}
and thus the dynamics of each species' abstract state information can be computed as
\begin{equation}
    \dot{\mathrm{x}}^{(s)}(t) = K^{(s)}\mathrm{x}^{(s)}(t), \quad \forall s = 1, 2, \cdots, S \label{eq:x_s_dynamics}
\end{equation}
where $K^{(s)} \in \mathbb{R}^{M \times M}_{+}$ is the rate matrix of species $s$, defined as follows
\begin{equation}
K^{(s)}_{ij} =
\begin{cases}
    k_{ji}^{(s)},& \text{if } i \neq j, (i,j) \in \mathcal{E} \\
    0,              & \text{if } i \neq j, (i,j) \notin \mathcal{E} \\
    -\sum_{i=1, (i,j) \in \mathcal{E}}^{M} k_{ij}^{(s)}, & \text{if } i=j
\end{cases}
\end{equation}
The solution of the dynamics in (\ref{eq:x_s_dynamics}) for $s^\mathrm{th}$ species can be written as
\begin{equation}
    \mathrm{x}^{(s)}(\tau) = e^{K^{(s)}\tau}\ \mathrm{x}^{(s)}(0) \label{eq:x_s_solution}
\end{equation}
Thus, the time evolution of the abstract state information is given by
\begin{equation}
    \bm{X}(\tau) = \sum_{s=1}^S e^{K^{(s)}\tau}\ z^{(s)}(0)
\end{equation}
where $z^{(s)}(0) = \bm{X}(0) \odot (\bm{1} \cdot e_s) \in \mathbb{N}^{M \times S}$, $\bm{1}$ is an $M$-dimensional vector of ones, and $e_s$ is the $S$-dimensional unit vector with its $s^\mathrm{th}$ element equal to one.

\subsection{Trait Aggregation and Distribution}\label{subsec:trait_dist}
Finally, we represent the trait distribution across the tasks by the \emph{trait distribution matrix} $\bm{Y}(t) \in \mathbb{R}_{+}^{M \times U}$ and is computed as
\begin{equation}
    \bm{Y}(t) = \bm{X}(t)Q \label{eq:trait_dist_tasks}
\end{equation}
Thus, each column of $\bm{Y}(t)$ represents the aggregated amounts of the corresponding trait available at each task at time $t$. Put another way, $\bm{Y}(t)$ represents the aggregation of various traits assigned to each task at time $t$.

Note that the stochastic nature of $Q$ results in the elements of $\bm{Y}(t)$ being random variables. The expected value of $\bm{Y}(t)$ can be computed as follows
\begin{equation}
    \bm{\mu}_Y(t) =  \bm{X}(t) \bm{\mu}_Q \label{eq:mean_trait_dist_tasks}
\end{equation}
and since the elements of $Q$ are independent random variables, the closed form expression for the variance of each element of $\bm{Y}$ can be captured in the following matrix
\begin{equation}
    \mathrm{Var}_{Y}(t) = (\bm{X}(t) \odot \bm{X}(t))\ \mathrm{Var}_Q \label{eq:var_trait_dist_tasks}
\end{equation}
where $\odot$ denotes the Hadamard (entry-wise) product. Furthermore, the covariance between any two elements of $Y$ is given by
\begin{equation}
    \mathrm{Cov}\{\bm{Y}_{ij}, \bm{Y}_{kl}\} =
    \begin{cases}
        \sum_{s=1}^{S} (x^{(s)}_{i}\ x^{(s)}_{k}\ \sigma^2_{sj}), & \text{if } j = l \\
        0, & \text{otherwise}
    \end{cases}
\end{equation}

\begin{tcolorbox}[colback=blue!5!white, enhanced jigsaw, breakable]
\textit{Example}: The expected value of the species-trait matrix and the initial abstract state information of our example problem are defined in (\ref{eq:ex_mu_Q_and_var_Q}) and (\ref{eq:ex_X0}), respectively. Thus, the expected values of the initial trait distribution and the associated variances can be computed using (\ref{eq:mean_trait_dist_tasks}) and (\ref{eq:var_trait_dist_tasks}), and are thus given by
\begin{equation}
    \bm{\mu}_Y(0) =
    \begin{bmatrix}
    1250 & 375 & 500 & 3500 \\
    3750 & 250 & 250 & 0 \\
    4375 & 0 & 625 & 1500 \\
    5000 & 875 & 750 & 3500 \\
    0 & 0 & 0 & 0
    \end{bmatrix}
    \mathrm{Var}_Y(0) =
    \begin{bmatrix}
    1875 & 625 & 937.5 & 3500 \\
    1250 & 937.5 & 312.5 & 0 \\
    625 & 0 & 1500 & 5437.5 \\
    3750 & 1437.5 & 2437.5 & 5750 \\
    0 & 0 & 0 & 0
    \end{bmatrix} \nonumber
\end{equation}
The above matrices explain how the team's capabilities are expected to be distributed across the different tasks.
\end{tcolorbox}

\subsection{Cumulative vs Non-Cumulative Traits}\label{subsec:cum_vs_non-cum}
Indeed, not all traits are suitable for aggregation. To handle this fact, we explicitly differentiate between two kinds of traits: \emph{cumulative} and \emph{non-cumulative} traits. While examples of cumulative traits include ammunition, equipment, and coverage area, those of non-cumulative traits include speed and special skills. We model cumulative traits as continuous variables (i.e., $q^{(s)}_i \in \mathbb{R}_{+}, \forall i \in \mathcal{C}$, where $\mathcal{C} \subseteq \{1,2,\cdots,U\}$ is the set of indices corresponding to cumulative traits). To handle non-cumulative traits, we approximate them as binary variables (i.e., $q^{(s)}_j \in \{0,1\}, \forall j \in \{1,2,\cdots,U\} \setminus \mathcal{C}$).

Unlike prior methods that only consider the existence of capabilities to define binary traits, STRATA assigns binary values to non-cumulative traits based on the following rule
\begin{equation}
    \mu_{si} =
    \begin{cases}
        1, & \text{if}\ \mu_{si} \geq q^{\text{min}}_i \\
        0, & \text{otherwise}
    \end{cases}
\end{equation}
where $q^{\text{min}}_i$ is a user-defined minimum acceptable value for the $i^\mathrm{th}$ trait. The binary representation of non-cumulative traits captures information about whether the agents of each species posses the minimum required capabilities. Further, when aggregated to form  $\bm{\mu}_Y$ (Section \ref{subsec:trait_dist}), the binary representation of non-cumulative traits provides the total number of agents that meet the minimum requirements. %Note that this binary representation is not pursued for the entries of variance matrix $\mathrm{Var}_Q$ that are associated with non-cumulative traits.

\begin{tcolorbox}[colback=blue!5!white, enhanced jigsaw, breakable]
\textit{Example}: Let us expand our example, which had four cumulative traits, to include a new non-cumulative trait: speed ($m/s$). Let the expected speed trait for each species be given by $\mu_{q^{(5)}} = [8, 2, 5, 10]$. In order to represent the non-cumulative speed trait in the binary form, let the minimum acceptable value for speed be $q_5^\text{min} = 7\ m/s$. The matrix denoting the new species-trait distribution will contain an additional column and is given by

\begin{equation}
    \bm{\mu}'_Q =
    \begin{bmatrix}
    50 & 15 & 20 & 140 & 1\\
    150 & 10 & 10 & 0 & 0\\
    175 & 0 & 25 & 60 & 0\\
    200 & 35 & 30 & 140 & 1\\
    \end{bmatrix} \label{eq:ex_new_mu_Q}
\end{equation}

Note that the average speeds of Species 2 ($2 m/s$) and 3 ($5 m/s$) are lower than the minimum requirement of $7 m/s$ miles. Thus, Species 2 and 3 are considered to not meet the minimum requirements for speed and are assigned zeros for the same trait.

\

Given an assignment $\bm{X}(0)$, similar to cumulative traits, the effects of the new non-cumulative trait can be observed by computing the new expected a task-trait distribution using (\ref{eq:mean_trait_dist_tasks}), and is thus given by

\begin{equation}
    \bm{\mu}'_Y(0) =
    \begin{bmatrix}
    1250 & 375 & 500 & 3500 & 25\\
    3750 & 250 & 250 & 0 & 0\\
    4375 & 0 & 625 & 1500 & 0\\
    5000 & 875 & 750 & 3500 & 25\\
    0 & 0 & 0 & 0 & 0
    \end{bmatrix}
    \label{eq:ex_mu_Y_0_2}
\end{equation}
Note that, while the first four columns of the task-trait distribution (that correspond to cumulative traits) remain unchanged, the last column (that corresponds to the non-cumulative speed trait) now represents the number of agents assigned to each task that meet the minimum requirement.
\end{tcolorbox}

\section{Problem Formulation}\label{sec:problem}
Based on the modeling framework described in Section \ref{sec:model}, this section considers the problem of task assignment that achieves a desired trait distribution across tasks. Specifically, we wish to find the transition rates $K^{(s)}$ for each species such that the expected trait distribution over tasks $\mu_{Y}(t)$, defined in (\ref{eq:trait_dist_tasks}), reaches the desired trait distribution $\bm{Y}^*$ as quickly as possible while minimizing the expected variance in the trait distribution.

We express the problem as the following optimization problem
\begin{align}
    \tau^*, K^{(s)^*} & = \arg \min_{\tau, K^{(s)}} \tau \label{eq:opt_prob} \\
    s.t. & \quad \bm{X}(\tau)\mu_Q \in \mathcal{G}(\bm{Y}^*) \label{eq:constraint}  \\
    & \Vert \mathrm{Var}_Y(\tau) \Vert_F \leq  \epsilon_{\mathrm{var}}  \label{eq:var_constraint}
\end{align}
where $\epsilon_{\mathrm{var}}$ is the threshold used to limit the variance in $\bm{Y}(\tau)$,  $\mathcal{G}(\bm{Y}^*): \mathbb{R}_{+}^{M \times U} \rightarrow \Omega$, named the goal function, is a function that defines the set of admissible expected trait distribution matrices $\Omega$. Note that the constraint in (\ref{eq:var_constraint}) helps minimize the expected variance, and thereby, maximize the chances that the actual trait distribution $\bm{Y}(\tau^*)$ meets the goal function.

We consider two goal functions:
\begin{enumerate}
\item \emph{Exact matching}: $\quad \mathcal{G}_1(\bm{Y}^*) = \{\mu_{Y}|\bm{Y}^* = \mu_Y\}$

\item \emph{Minimum matching}: $\quad \mathcal{G}_2(\bm{Y}^*) = \{\mu_Y|\bm{Y}^* \preceq \mu_Y\}$
\end{enumerate}
where $\preceq$ denotes the element-wise less-than-or-equal-to operator. While goal function $\mathcal{G}_1$ requires achieving the exact desired trait distribution, goal function $\mathcal{G}_2$ requires the trait distribution be greater than or equal to the desired trait distribution. In other words, $\mathcal{G}_1$ does not allow any deviation from the desired trait distribution, and $\mathcal{G}_2$ allows for over-provisioning.

\begin{tcolorbox}[colback=blue!5!white, enhanced jigsaw, breakable]
\textit{Example}: Let the desired trait distribution for our example be given by
\begin{equation}
    \bm{Y}^* =
    \begin{bmatrix}
    0 & 0 & 0 & 0 & 0\\
    1250 & 375 & 500 & 3500 & 25\\
    3750 & 250 & 250 & 0 & 0\\
    4375 & 0 & 625 & 1500 & 0\\
    5000 & 675 & 750 & 3500 & 25
    \end{bmatrix}.
\end{equation}
Note that the expected initial trait distribution $\mu'_Y(0)$ is defined in (\ref{eq:ex_mu_Y_0_2}). Thus, the task assignment algorithm is required to compute the transition rates $K^{(s)^*}$ such that the team's expected trait distribution satisfies the chosen goal function as quickly as possible.
\end{tcolorbox}

\section{Diversity Measures}
Large heterogeneous teams with multiple species might result in capabilities that are complementary and or redundant. We study the properties of the average species-trait matrix $\mu_Q$ to understand the similarities and variations among the species of a given team. Measures of team diversity were defined in \cite{prorok2017impact} for species defined by binary traits. In this section, we define diversity measures for species defined by continuous traits. We define two measures of trait diversity for a given team, one for each of the two goal functions defined in Section \ref{sec:problem}. To this end, we utilize the following definitions.

\begin{definition}
\emph{Minspecies}: In a team described by an average species-trait matrix $\mu_Q$, a \emph{minspecies} set is a subset of rows of $\mu_Q$ with minimal cardinality, such that the system can still achieve the goal $\mathcal{G}(\bm{Y}^*)$.
\end{definition}

\begin{definition}
\emph{Minspecies cardinality}: The cardinality of the minspecies set is defined as the Minspecies cardinality and is represented by the function $\mathcal{D_G}:\mathbb{R}_+^{S \times U} \rightarrow \mathbb{N}_+$ that takes the average species-trait matrix $\mu_Q$ as the input and returns the minimum number of species required to achieve the goal $\mathcal{G}(\bm{Y}^*)$.
\end{definition}

\subsection{Eigenspecies}
First, we define a diversity measure related to the exact matching goal, $\mathcal{G}_1$.

\begin{proposition} The cardinality of eigenspecies (the minspecies corresponding to goal function $\mathcal{G}_1$) is computed as follows
\begin{align}
    \mathcal{D}_{\mathcal{G}_1} = & \min \vert \mathcal{M}_1 \vert \\
    \mathrm{s.t.} & \sum_{s \in \mathcal{M}_1} \alpha_{s\tilde{s}} \mu_{q^{(s)}} =  \mu_{q^{(\tilde{s})}},\ \forall \tilde{s} \notin \mathcal{M}_1, \forall \alpha_{s\tilde{s}} \in \mathbb{N}
\end{align}
where $\mathcal{M}_1$ is a subset of all the species in the team, $\vert \cdot \vert$ denotes the cardinality, and $\mathbb{N}$ is the set of all non-negative integers.
\end{proposition}

\begin{proof}
The expected species-trait matrix can be factorized as $\mu_Q = A\hat{\mu}_Q$, such that $\mu_q = A\hat{\mu}_Q$, where $A \in \mathbb{N}^{S \times \vert \mathcal{M}_1 \vert}$ and $\hat{\mu}_Q \in \mathbb{R}^{\vert \mathcal{M}_1 \vert \times U}$. Now, $\bm{Y}^* = \bm{X}^*\mu_Q = \bm{X}^*A\hat{\mu}_Q = \hat{\bm{X}}\hat{\mu}_Q$ where $\hat{\bm{X}} = \bm{X}^*A$. Thus, there exists an agent distribution $\hat{\bm{X}}$ that can achieve the goal $\mathcal{G}_1$ with only a subset of the species, defined using the minimal species-trait matrix $\hat{\mu}_Q$.
\end{proof}

Thus, $\mathcal{M}_1$ contains the minimal set of species that can exactly match any desired trait distribution without recruiting agents from species not in $\mathcal{M}_1$, and $\mathcal{D}_{\mathcal{G}_1}$ denotes the number of species that form $\mathcal{M}_1$. Note that weighting factors of the sum are restricted to be natural numbers. The motivation behind this restriction is that, when aggregating traits, the weighting factor corresponds to the number of agents we are considering when aggregating traits.

\begin{tcolorbox}[colback=blue!5!white, enhanced jigsaw, breakable]
\textit{Example}: For the team in our example, the average species-trait matrix is defined in (\ref{eq:ex_mu_Q_and_var_Q}). Note that the sum of the first row rows is equal to the last row. Further, no other rows are equal to the weighted (by natural numbers) sum of the remaining rows. Specifically, $\mu_{q^{(1)}} + \mu_{q^{(2)}} = \mu_{q^{(4)}}$. The average species-trait matrix can thus be factorized as $\mu_Q = A\hat{\mu}_Q$, where
\begin{equation}
    A =
    \begin{bmatrix}
    1 & 0 & 0 \\
    0 & 1 & 0 \\
    0 & 0 & 1 \\
    1 & 1 & 0 \\
    \end{bmatrix}
    \
    \hat{\mu}_Q =
    \begin{bmatrix}
    50 & 15 & 20 & 140 & 1\\
    150 & 10 & 10 & 0 & 0\\
    175 & 0 & 25 & 60 & 0\\
    \end{bmatrix}
\end{equation}
Thus, $\mathcal{M}_1=\{1, 2, 3\}$ and consequently $\mathcal{D}_{\mathcal{G}_1} = 3$. In other words, the traits of only one species (Species 4) can be exactly matched by aggregating the traits of other species (Species 1 and 2).
\end{tcolorbox}

\subsection{Coverspecies}
Next, we define a diversity measure for the minimum matching goal, $\mathcal{G}_2$.

\begin{proposition} The cardinality of coverspecies (the minspecies corresponding to goal function $\mathcal{G}_2$) is computed as follows
\begin{align}
    \mathcal{D}_{\mathcal{G}_2} = & \min \vert \mathcal{M}_2 \vert \\
     \mathrm{s.t.} & \sum_{s \in \mathcal{M}_2} \alpha_{s\tilde{s}} \mu_{q^{(s)}} \succeq  \mu_{q^{(\tilde{s})}},\ \forall \tilde{s} \notin \mathcal{M}_2, \forall \alpha_{s\tilde{s}} \in \mathbb{N}
\end{align}
\end{proposition}

\begin{proof}
The expected species-trait matrix can be factorized as $\mu_Q = A\hat{\mu}_Q$, such that $\mu_Q \preceq A\hat{\mu}_Q$, where $A \in \mathbb{N}^{S \times \vert \mathcal{M}_2 \vert}$, $\hat{\mu}_Q \in \mathbb{R}^{\vert \mathcal{M}_2 \vert \times U}$. Now, $\bm{Y}^* \preceq \bm{X}^*\mu_Q \preceq \bm{X}^*A\hat{\mu}_Q$. Thus, there exists an agent distribution $\hat{\bm{X}} = \bm{X}^*A$ that can achieve the goal $\mathcal{G}_2$ with only a subset of the species, defined using the species-trait matrix $\hat{\mu}_Q$.
\end{proof}

Thus, $\mathcal{M}_2$ contains the minimal set of species that can satisfy (with potential over-provision) any desired trait distribution without recruiting agents from species not in $\mathcal{M}_2$, and $\mathcal{D}_{\mathcal{G}_2}$ is the number of species that form such a minimal set.

\begin{tcolorbox}[colback=blue!5!white, enhanced jigsaw, breakable]
\textit{Example}: For the team in our example, the average species-trait matrix is defined in (\ref{eq:ex_mu_Q_and_var_Q}). Note that each element of the last row is larger or equal to the corresponding elements in every other row. The average species-trait matrix can thus be factorized as $\mu_Q \preceq A\hat{\mu}_Q$, where
\begin{equation}
    A =
    \begin{bmatrix}
    1 \\
    1 \\
    1 \\
    1 \\
    \end{bmatrix}
    \quad
    \hat{\mu}_Q =
    \begin{bmatrix}
    200 & 35 & 30 & 140 & 1
    \end{bmatrix}
\end{equation}
Thus, $\mathcal{M}_2=\{4\}$ and consequently $\mathcal{D}_{\mathcal{G}_2} = 1$.  In other words, the traits of three species (Species 1,2, and 3) can be minimum matched by the traits of one species (Species 4).
\end{tcolorbox}

\section{Solution Approach}

This section details the proposed solution to the optimization problem defined in (\ref{eq:opt_prob})-(\ref{eq:constraint}). Our solution computes the transition rates $K^{(s)^*}$ without assuming knowledge of the optimal agent distribution $\bm{X}^*$.

\subsection{Optimization Criteria}\label{subsec:opt_criteria}
We begin by considering the time evolution of average trait distribution over the tasks. To this end, we combine (\ref{eq:mean_trait_dist_tasks}) and (\ref{eq:x_s_solution}), yielding
\begin{equation}
    \bm{\mu}_Y (\tau, K^{(1,..,S)}, \mathrm{x}^{(s)}(0)) = \sum_{s=1}^S e^{K^{(s)}\tau}\ \mathrm{x}^{(s)}(0)\ \mu_{q^{(s)}}
\end{equation}
In order to satisfy the goal function constraint, as defined in (\ref{eq:constraint}), we impose constraints on two error functions. The first error function computes the trait distribution error and is defined separately for each goal function as follows:
\begin{align}
    E_1^{\mathcal{G}_1}(\tau, K^{(1,..,S)}, \bm{X}(0)) = & \Vert Y^* - \bm{\mu}_Y(\tau) \Vert_{F}^{2}\\
    E_1^{\mathcal{G}_2}(\tau, K^{(1,..,S)}, \bm{X}(0)) = & \Vert \max [(Y^* - \bm{\mu}_Y(\tau)),\ 0] \Vert_{F}^{2}
\end{align}
where $\Vert \cdot \Vert_F$ denotes the Frobenius norm of a matrix. Note that we have omitted the dependence of $\mu_Y(\cdot)$ on the transition rates and initial conditions for brevity. The second error function measures the deviation from the steady state agent trait distribution and is defined as follows for both goal functions:
\begin{align}
    E_2(\tau, K^{(1,2,..,S)}, \bm{X}(0)) = &\sum_{s=1}^S \Vert e^{K^{(s)}\tau}\ \mathrm{x}^{(s)}(0)\ \\ \nonumber
    &\quad -\ e^{K^{(s)}(\tau+\nu)}\ \mathrm{x}^{(s)}(0) \Vert_2^2
\end{align}
The first error function $E_1$ (for both goal functions) penalizes the system when the trait distribution  at time $\tau$ does not satisfy the appropriate goal, and the error function $E_2$ penalizes the system if its trait distribution does not reach steady state at time $\tau$. Thus, enforcing upper bounds on these error functions guarantees a certain minimum level of performance. Further, as noted in (\ref{eq:var_constraint}), the expected variation in the achieved trait distribution, $\mathrm{Var}_Y(\tau)$, is also considered in the optimization.

\subsection{Optimization Problem}
Based on the definitions in Sections \ref{subsec:opt_criteria}, we reformulate the optimization problem in (\ref{eq:opt_prob})-(\ref{eq:var_constraint}) for goal $\mathcal{G}_1$ as follows
\begin{align}
        \tau^*, K^{(s)^*} & = \arg \min_{\tau, K^{(s)}} \tau \label{eq:mod_opt_prob} \\
    s.t.\ & E_1^{\mathcal{G}_1}(\tau, K^{(1,..,S)}, \bm{X}(0)) \leq \epsilon_1 \label{eq:E_1_constraint} \\
         & E_2(\tau, K^{(1,..,S)}, \bm{X}(0)) \leq \epsilon_2 \label{eq:E_2_constraint} \\
         & \Vert \mathrm{Var}_Y(\tau,K^{(1,..,S)}, \bm{X}(0)) \Vert_{F}^{2} \leq \epsilon_{\mathrm{var}} \label{eq:Var_Y_constraint} \\
         & k_{ij}^{(s)} \leq k_{ij,\text{max}}^{(s)},\ \forall i,j=\{1,..,M\},\ \forall s=\{1,..,S\} \label{eq:max_rate_constraint}  \\
         & \tau > 0 \label{eq:time_constraint}
\end{align}
where $\epsilon_1$, $\epsilon_2$, and $\epsilon_{\mathrm{var}}$ are user-defined positive scalars. Note that the optimization problem for goal $\mathcal{G}_2$ is identical except that we replace the constraint in (\ref{eq:E_1_constraint}) with $E_1^{\mathcal{G}_2}(\tau, K^{(1,..,S)}, \bm{X}(0)) \leq \epsilon_1$.

Note that the solution to the optimization problem in (\ref{eq:mod_opt_prob})-(\ref{eq:time_constraint}) guarantees minimum levels of performance, both in terms of achieving and maintaining the appropriate goal, as defined by the arbitrary constants $\epsilon_1$ and $\epsilon_2$, respectively. The constraint in (\ref{eq:Var_Y_constraint}) helps ensure that the expected variance of the achieved trait distribution is below a desired threshold. Thus, for each task, the constraint in (\ref{eq:Var_Y_constraint}) encourages the system to recruit agents who possess traits (required for the task) with relatively low variance, there by increasing the odds of the actual trait distribution satisfying the specified goal.

\subsection{Analytical Gradients}
To efficiently solve the the optimization problem in (\ref{eq:mod_opt_prob})-(\ref{eq:time_constraint}), we derive and utilize the analytical gradients of all the constraints with respect to the decision variables. In this section we define the analytical gradients of constraints defined in (\ref{eq:E_1_constraint})-(\ref{eq:time_constraint}) with respect to the unknowns $\tau$ and $K^{(s)}$. We refer the readers to \cite{prorok2017impact} for closed-form expressions of the derivatives of $E_1^{\mathcal{G}_1}$, $E_1^{\mathcal{G}_2}$ and $E_2$ with respect to both $K^{(s)}$ and $\tau$.

We adapt the closed-form expressions for derivatives of the matrix exponential \cite{kalbfleisch1985analysis}, and use the chain rule to derive the derivatives of $\Vert \mathrm{Var}_Y \Vert_{F}^{2}$ with respect $K^{(s)}$ and $\tau$ as follows \footnote{we drop the arguments of $\mathrm{Var}_Y$ for brevity}
\begin{equation}
     \frac{\partial \Vert \mathrm{Var}_Y \Vert_{F}^{2}}{\partial K^{(s)}} = \frac{\partial \Vert \mathrm{Var}_Y \Vert_{F}^{2}}{\partial e^{K^{(s)}\tau}}\ \frac{\partial e^{K^{(s)}\tau}}{\partial K^{(s)}\tau}\ \frac{\partial K^{(s)}\tau}{\partial K^{(s)}}
\end{equation}
\begin{equation}
     \frac{\partial \Vert \mathrm{Var}_Y \Vert_{F}^{2}}{\partial \tau} = \frac{\partial \Vert \mathrm{Var}_Y \Vert_{F}^{2}}{\partial e^{K^{(s)}\tau}}\ \frac{\partial e^{K^{(s)}\tau}}{\partial K^{(s)}\tau}\ \frac{\partial K^{(s)}\tau}{\partial \tau}
\end{equation}
where
\begin{equation}
    \frac{\partial \Vert \mathrm{Var}_Y \Vert_{F}^{2}}{\partial e^{K^{(s)}\tau}} = 4 \big[ (\mathrm{Var}_Y\mathrm{Var}_Q^T) \odot X(t) \big] z^{(s)}(0)
\end{equation}
Thus, the closed form expressions of the derivatives are given by
\begin{equation}
    \frac{\partial \Vert \mathrm{Var}_Y \Vert_{F}^{2}}{\partial K^{(s)}} = (V^{(s)})^{-T} B^{(s)} (V^{(s)})^{T} \tau
\end{equation}
\begin{equation}
    \frac{\partial \Vert \mathrm{Var}_Y \Vert_{F}^{2}}{\partial \tau} = \sum_{s=1}^{S} 1^T \big[ (V^{(s)})^{-T} B^{(s)} (V^{(s)})^{T} K^{(s)} \big] 1
\end{equation}
where $K^{(s)} = V^{(s)} D^{(s)} (V^{(s)})^{-1}$ is the eigenvalue decomposition of $K^{(s)}$, $D^{(s)} = \mathrm{diag}(d_1, \cdots , d_M)$ is the diagonal matrix with the eigenvalues of $K^{(s)}$, $B^{(s)}$ is a $M \times M$ matrix defined as

\begin{equation}
    B^{(s)} = \bigg[ (V^{(s)})^T \frac{\partial \Vert \mathrm{Var}_Y \Vert_{F}^{2}}{\partial e^{K^{(s)}\tau}} (V^{(s)})^{-T} \odot W^{(s)} \bigg]
\end{equation}
and $W(\tau)$ is a $M \times M$ matrix with its $kl^\mathrm{th}$ element is given by
\begin{equation}
    W_{kl}^{(s)} =
    \begin{cases}
        \frac{(e^{d_k \tau} - e^{d_l \tau})}{d_k \tau - d_l\tau}, & k \neq l \\
        e^{d_k \tau}, & k=l
    \end{cases}
\end{equation}

\subsection{Decentralized Online Implementation}
\hl{We note that it is possible to realize a decentralized implementation of the proposed optimization problem in ({\ref{eq:mod_opt_prob}})-({\ref{eq:time_constraint}}). Indeed, prior work has demonstrated that one can compute the state information $X$ locally, and enable online adaptation (re-optimization) of the transition rates as the state information is updated} {\cite{milam2005receding,prorok2017impact,bandyopadhyay2017probabilistic,jang2018local}}. \hl{Our approach, while introducing new capabilities, directly inherits these advantages from prior work.}

\section{Experimental Evaluation}\label{sec:exps}

We evaluate STRATA using two sets of experiments. In both experiments we compare STRATA's performance with that of a baseline. Our baseline method is a bootstrapped version of the binary-trait-based method introduced in \cite{prorok2017impact}\footnote{source code acquired from https://github.com/amandaprorok/diversity.git}.

\subsection{\hl{Baseline and Metrics}}

Since the baseline algorithm~\cite{prorok2017impact} requires the desired trait distribution to be specified in the binary trait space and only considers binary species-trait distributions, we make the following modifications to the baseline. We define a binary species-trait matrix to be $\bar{Q} = \mathrm{sign}(\mu_Q)$, where $\mathrm{sign}(\cdot)$ is the signum operator applied to each element of its matrix argument. We also define a modified desired trait distribution for the baseline: $\bar{Y} = \floor*{Y^* \oslash \mu_Y}$, where $\floor{\cdot}$ is the floor function applied to each element of a matrix, $\oslash$ refers to Hadamard (element-wise) division, $\mu_Y = [\mu_Y^1 \odot 1_M, \cdots, \mu_Y^U \odot 1_M]$, $\mu_Y^i$ is the mean desired value of the $i^\mathrm{th}$ trait computed across all species, and $1_M$ is a $M$-dimensional vector of ones. \hl{We implemented both the baseline and STRATA on a computer running Intel Core i7-4770K with 16 GB of memory.}

\begin{figure*}[ hbt!]
    \centering
    \includegraphics[width=0.7\textwidth] {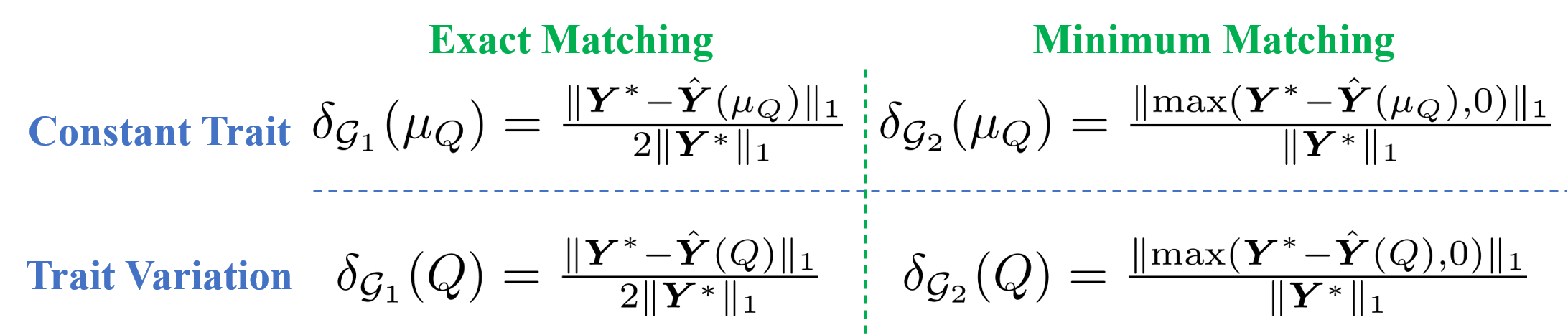}
    \caption{The four error measures used to quantify the proportion of trait mismatch.}
    \label{fig:metrics}
\end{figure*}

The task assignment performance of each method is evaluated in terms of four measures of percentage trait mismatch, as defined in Fig. \ref{fig:metrics}. As illustrated, the metrics measure performance in two scenarios: when the actual traits of the agents are (1) assumed to be known and equal to the average of the corresponding trait (\textit{constant trait}), and (2) assumed to be unknown and sampled based on the trait distribution (\textit{trait variation}). Additionally, in each scenario, the performance is measured in terms of both \textit{exact} and \textit{minimum} trait matching, irrespective of the optimization goal.

\begin{figure*}[htb!]
    \centering
    \includegraphics[trim={3cm 4.8cm 3cm 5.6cm}, clip, width=0.8\textwidth] {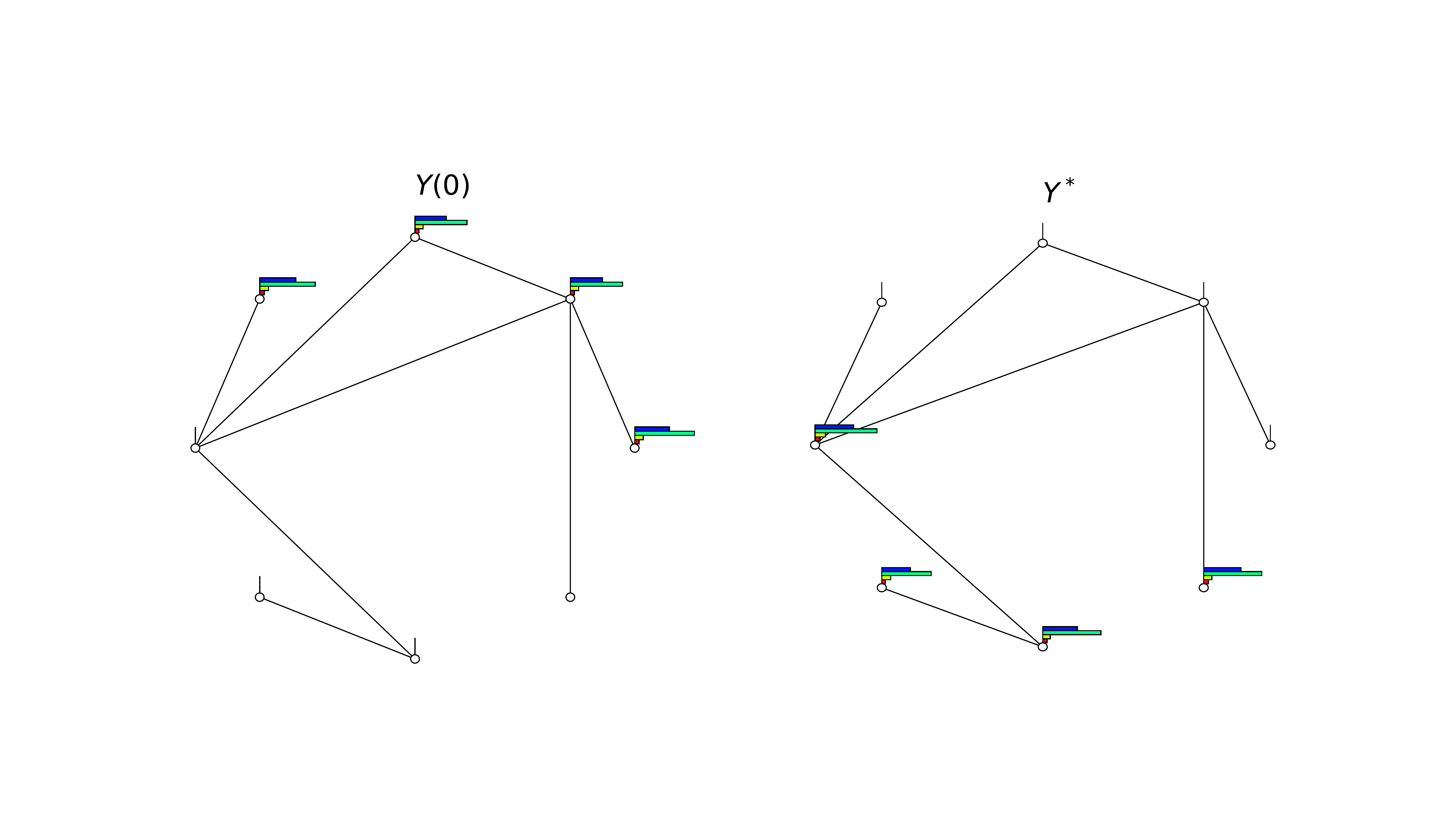}
    \caption{The initial (left) and desired (right) team configurations from an exemplary run with eight tasks (nodes) and four traits. The bar plots denote the trait distribution at each task and the edges represent the possibility of switching between the corresponding tasks.}
    \label{fig:initial_and_desired_graph}
\end{figure*}

\subsection{Simulation}\label{subsec:sim}
In the first set of experiments, we study the performances of STRATA and the baseline in terms of matching the desired trait requirements for a large heterogeneous team. To this end, we simulate a task assignment problem with $M = 8$ nodes (tasks), $U = 5$ traits (3 cumulative and 2 non cumulative traits), and $S = 5$ species (each with 200 agents). We present the results computed from 100 independent simulation runs.

In each run, we make the following design choices. The task graph along with its connections is randomly generated. The initial and final agent distributions, $\bm{X}(0)$ and $\bm{X}^*$, are uniformly randomly generated. Based on the obtained $\bm{X}^*$, a desired trait distribution $\bm{Y}^*$ is computed for each run. The expected value of the species-trait matrix is chosen to be $\mu_{Q} = [a, a, a, b, b]^T$, where each element of $a \in \mathbb{R}_{+}^U$ is sampled from a uniform distribution: $a_i \sim \mathcal{U}(0,10)$, and each element of  $U$-dimensional $b$ is sampled from a discrete uniform distribution: $b_i \sim \mathcal{U}\{0,1\}$. Each element of $\mathrm{Var}_Q$ is sampled from a uniform distribution: $\mathcal{U}(0,2)$. An example initial and desired trait distribution is illustrated in Fig. \ref{fig:initial_and_desired_graph}. \hl{The maximum transition rates $k_{ij,\text{max}}^{(s)}, \forall i,j,s$ are chosen to be $0.02\ S^{-1}$. The thresholds $\epsilon_1$ and $\epsilon_2$ are both chosen so as to be equivalent to $5\%$ of $\Vert Y^* \Vert_F$.}

To ensure a fair comparison, we limit both STRATA and the baseline framework to a maximum of 20 meta iterations of the basin hopping algorithm during each run. In order to measure $\delta_{\mathcal{G}_1} (Q)$ and $\delta_{\mathcal{G}_2} (Q)$ for each run, 10 samples of the trait-species matrix $Q$ are generated to compute $\hat{\bm{Y}}(Q)$.

\begin{figure*}[htb!]
    \centering
    \includegraphics[width=\textwidth]{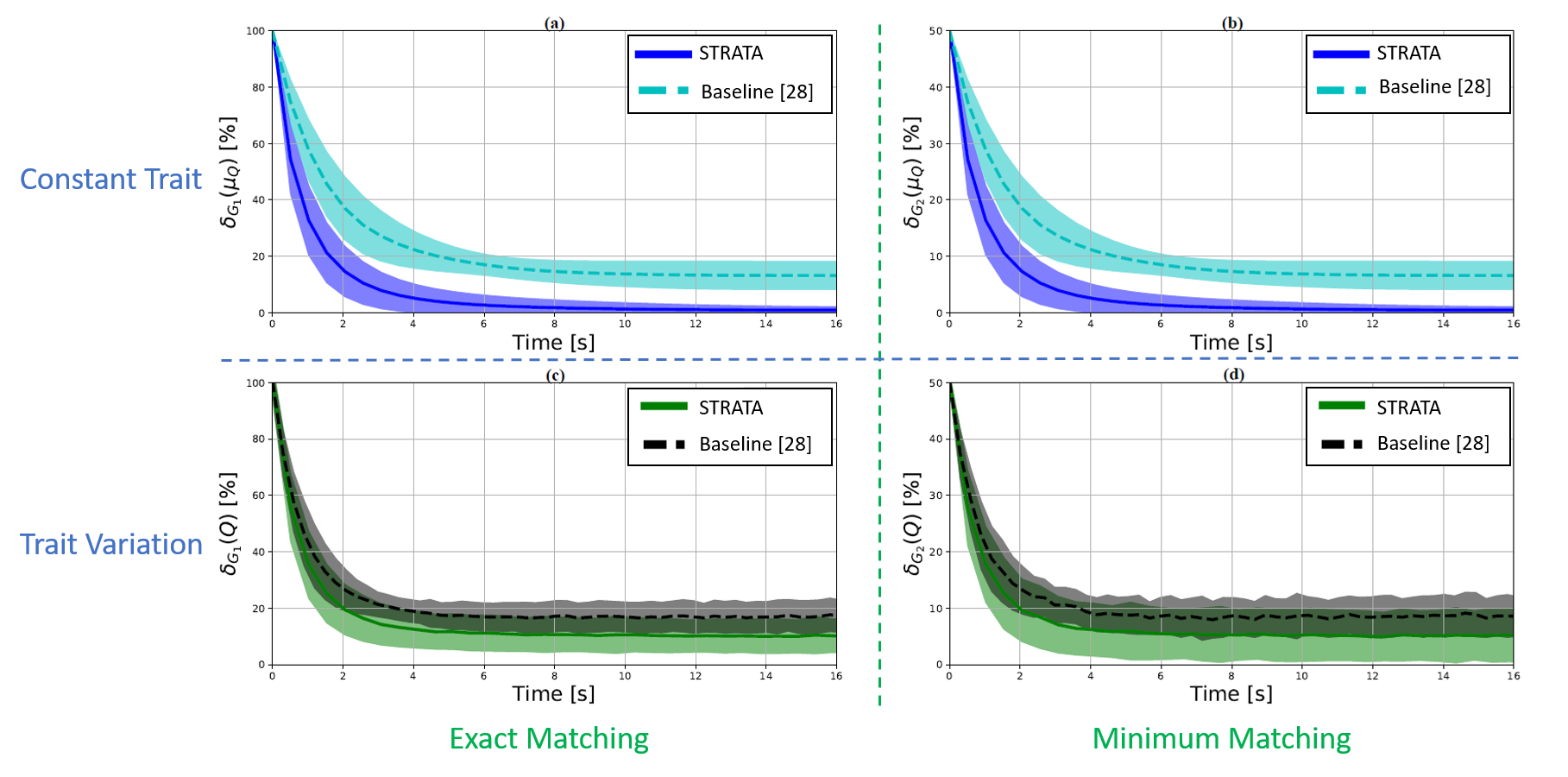}
    \caption{Comparison of the performances of STRATA and the baseline \cite{prorok2017impact} (binary) framework when optimizing for exact matching ($\mathcal{G}_1$). The performance of each framework is quantified in terms of four measures of percentage trait mismatch.}
    \label{fig:exact_comparison}
\end{figure*}

\begin{figure*}[htb!]
    \centering
    \includegraphics[width=\textwidth]{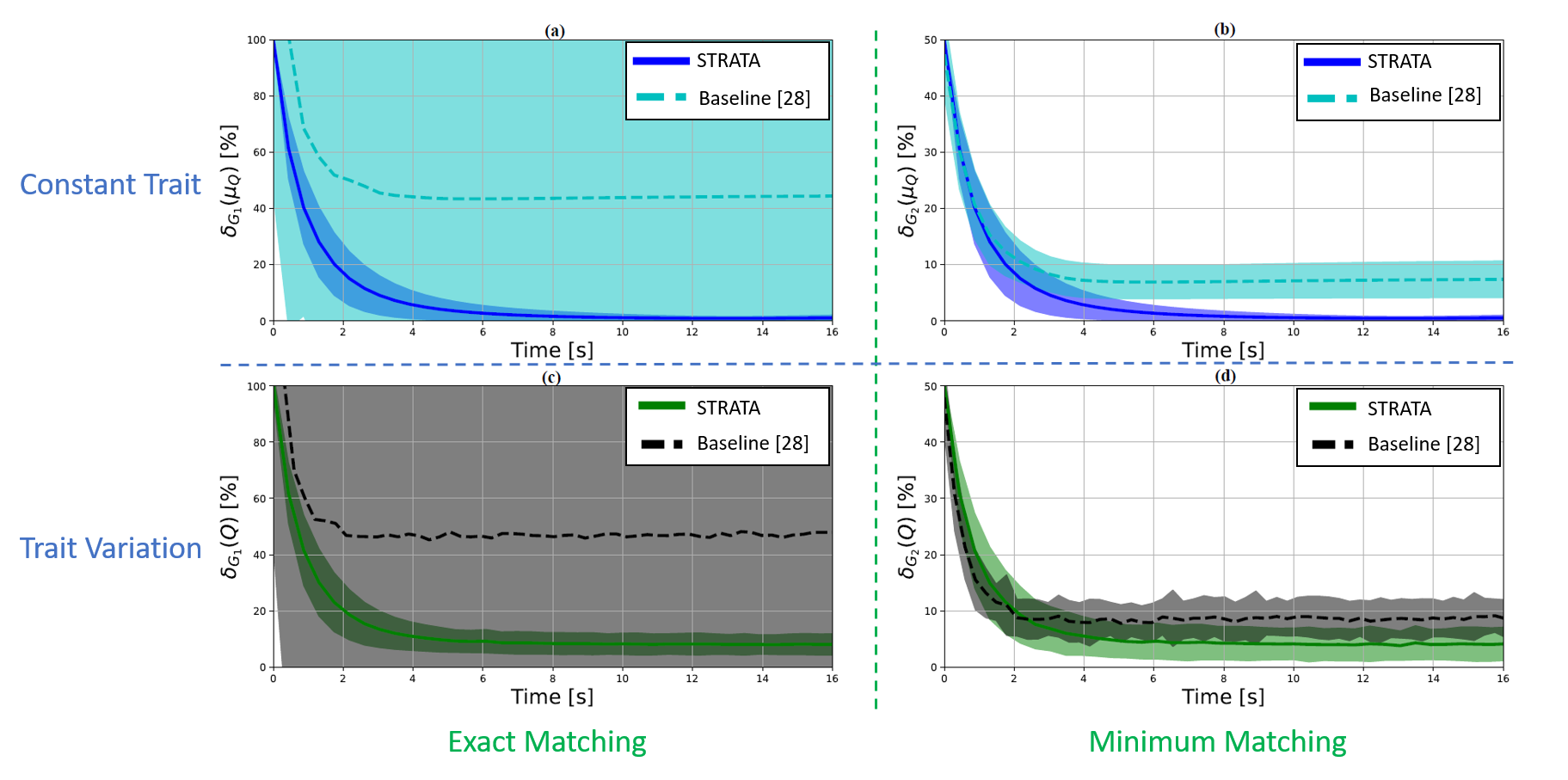}
    \caption{Comparison of the performances of STRATA and the baseline \cite{prorok2017impact} (binary) framework when optimizing for minimum trait matching ($\mathcal{G}_2$). The performance of each framework is quantified in terms of four measures of percentage trait mismatch.}
    \label{fig:min_comparison}
\end{figure*}

\textit{\textbf{Exact Trait Matching}}: First, we compute the transition rates according to both STRATA and the binary trait framework \cite{prorok2017impact} with respect to the function Goal $\mathcal{G}_1$. STRATA is found to converge during 79 of the 100 simulation runs and the binary trait framework during 10 runs. In Fig \ref{fig:exact_comparison}, we present the performances of both frameworks by plotting the errors (defined in Fig. \ref{fig:metrics}) as functions of time. Note that the error plots for each method reflect the error measures computed only across the converged runs.

As shown in Fig. \ref{fig:exact_comparison}(a) and (b), STRATA consistently performs better than the baseline in terms of deterministic performance, as measured by both $\delta_{\mathcal{G}_1} (\mu_Q)$ and $\delta_{\mathcal{G}_2} (\mu_Q)$. Further, as shown in Fig. \ref{fig:exact_comparison}(c) and (d), when the agents' traits are randomly sampled, STRATA performs better than the baseline on average. The stochastic nature of the species-trait matrix forces the errors to be larger than zero.

\begin{figure}
\centering
\includegraphics[width=0.5\columnwidth]{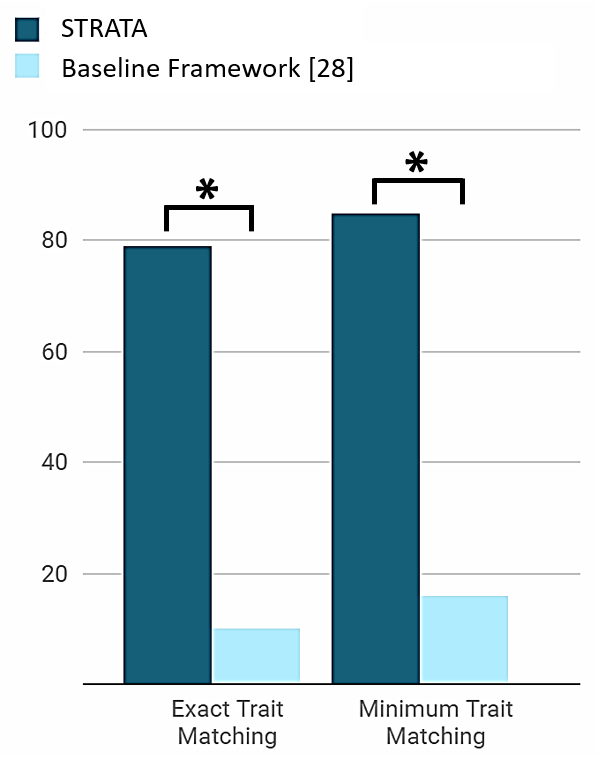}
\caption{\label{fig:converged_runs}Number of converged runs for each algorithm out of 100 simulation runs.}
\end{figure}

\textit{\textbf{Minimum Trait Matching}}: Next, we compute the transition rates according to both STRATA and the binary trait model \cite{prorok2017impact} with respect to the function Goal $\mathcal{G}_2$. STRATA is found to converge during 85 of the 100 simulation runs and the binary trait framework during 16 runs. In Fig \ref{fig:min_comparison}, we present the performances of both frameworks by plotting the errors (defined in Fig. \ref{fig:metrics}) as functions of time.

STRATA consistently performs better than the baseline when optimizing to satisfy minimum trait distribution, as measured by $\delta_{\mathcal{G}_2}(\mu_Q)$. On average, STRATA performs better than the baseline when considering stochastic species-trait matrix, as measured by $\delta_{\mathcal{G}_2}(Q)$. These assertions are supported by the plots in Fig. \ref{fig:min_comparison}(b) and (d).  In Fig. \ref{fig:min_comparison}(a) and (c), the baseline exhibits high error and variance in terms of both $\delta_{\mathcal{G}_1}(\mu_Q)$ and $\delta_{\mathcal{G}_1}(Q)$. This implies that when optimizes for $\mathcal{G}_2$, the binary trait model, unlike STRATA, results in a high level of over-provisioning.

\textit{\textbf{\hl{Discussion}}}: \hl{As demonstrated by the results above, reasoning about stochastic traits results in consistently satisfying complex task requirements. Stochastic trait models outperform binary trait models for a number a reasons. Firstly, binary trait models are incapable of reasoning about requirements in the continuous trait space. Secondly, to construct the modified (binary) trait distribution $\bar{Y}$, one is required to consider, at minimum, the average value of each trait in the team. In the process, however, the binary trait model ignores all variations both at the species and individual levels. Indeed, the advantages of considering these variations are reflected in terms of more accurate assignment of agents in Figs. {\ref{fig:exact_comparison}} and {\ref{fig:min_comparison}}}.

\hl{Given that the results above reflect only the performance from converged runs, it is important to investigate the ratio of converged runs to total number of runs. As seen in Fig. {\ref{fig:converged_runs}}, STRATA successfully converged to a solution in significantly ($p<0.001$) more runs than the binary trait framework for both exact trait matching ($\mathcal{G}_1$) and minimum trait matching ($\mathcal{G}_2$). This observation demonstrates that considering stochastic and continuous trait models over binary models is considerably more likely to satisfy complex trait requirements.}

\begin{table*}[bp]
\centering
\resizebox{0.9\textwidth}{!}{%
\begin{tabular}{cccccc}
\hline
 & \# Species & \begin{tabular}[c]{@{}c@{}}\# agents \\ per Species\end{tabular} & \# Tasks & \# Traits & Role Assignment \\ \hline
Team A & 4 & 3 & 3 & 4 & STRATA \\ \hline
Team B & 4 & 3 & 3 & 4 & Baseline \cite{prorok2017impact} \\ \hline
Team C & 4 & 3 & 3 & 4 & Random \\ \hline
\\
\end{tabular}%
}
\caption{Specifications of the teams implemented in the capture the flag environment.}
\label{CtF_table}
\end{table*}

\subsection{Capture the Flag}\label{subsec:ctf}

In this section, we study the question: Does STRATA improve higher-level team performance? To this end, we quantify the effect of STRATA on team performance in a capture the flag (CTF) game. \hl{Note that we do not explicitly model the adversarial elements of the game, and focus only on effectively assigning agents to roles, based on an empirically-identified ideal trait distribution $Y^*$}. We built the environment using the Unity 3D game engine.  We compare the performances of three teams, named A, B, and C. The details pertaining to each team are listed in Table \ref{CtF_table}. The three tasks (roles) in the game are \emph{defend}, \emph{attack}, and \emph{heal}, in that order. The four traits are \emph{speed}, \emph{viewing distance}, \emph{health}, and \emph{ammunition}, in that order. We consider speed and viewing distance are non-cumulative traits, and health and ammunition as cumulative traits. \hl{Similar to the experiments in simulation, the maximum transition rates $k_{ij,\text{max}}^{(s)}, \forall i,j,s$ are chosen to be $0.02\ S^{-1}$, and the thresholds $\epsilon_1$ and $\epsilon_2$ are both chosen so as to be equivalent to $5\%$ of $\Vert Y^* \Vert_F$.} \hl{The average time to optimize the transition rates for the game with $M = 3$, $U = 4$, and $S = 4$ is around $0.88\ ms$ for the baseline and $0.97\ ms$ for STRATA}.

The baseline task assignment strategy, similar to the experiment in Section \ref{subsec:sim}, is a bootstrapped version of the binary-trait-based method introduced in \cite{prorok2017impact}. For the random task assignment strategy, each agent is assigned uniformly randomly to one of the three roles. Note that both the algorithms are provided with identical teams, consisting of 12 agents. Thus, any variation in performance is limited to the task assignment strategy used by each team and the inherent randomness of the game.

The traits of the agents are sampled from the following stochastic species-trait matrix
\begin{equation}
    \mu_Q =
    \begin{bmatrix}
    1.5 & 15 & 90 & 40 \\
    1.5 & 30 & 60 & 40 \\
    3 & 15 & 80 & 30 \\
    3 & 30 & 350 & 30 \\
    \end{bmatrix}
    \
    \mathrm{Var}_Q =
    \begin{bmatrix}
    0.35 & 5 & 10 & 3 \\
    0.35 & 5 & 10 & 3 \\
    0.35 & 5 & 10 & 3 \\
    0.35 & 5 & 10 & 3 \\
    \end{bmatrix} \nonumber
\end{equation}
The minimum acceptable value for the non-cumulative traits are chosen to be as follows: $q^{(1)}_{min} = 0 m/s$ for speed and $q^{(2)}_{min} = 10 m$ for viewing distance. The desired trait distribution is designed to be
\begin{equation}
    \bm{Y}^* =
    \begin{bmatrix}
    2 & 2 & 120 & 80 \\
    6 & 6 & 340 & 200 \\
    4 & 4 & 320 & 140 \\
    \end{bmatrix} \nonumber
\end{equation}

All games are played with two teams at a time, one versus another. We compare the performances of each team  against the other two teams. We consider a team to have won a game if the team captures the opponent's flag and brings it back to the starting position. If neither team is able to capture and bring back the opponent's flag in 120 seconds, then the team with the highest number of active agents is considered the winner. Lastly, if both teams retain the same number of active agents after 120 seconds, the game is considered to have ended in a draw. The relative performances of all three approaches are illustrated in Fig. \ref{fig:ctf_wins}.

\begin{figure*}[htb!]
    \centering
    \includegraphics[trim={0cm 5cm 0cm 3.5cm}, clip, width=\textwidth]{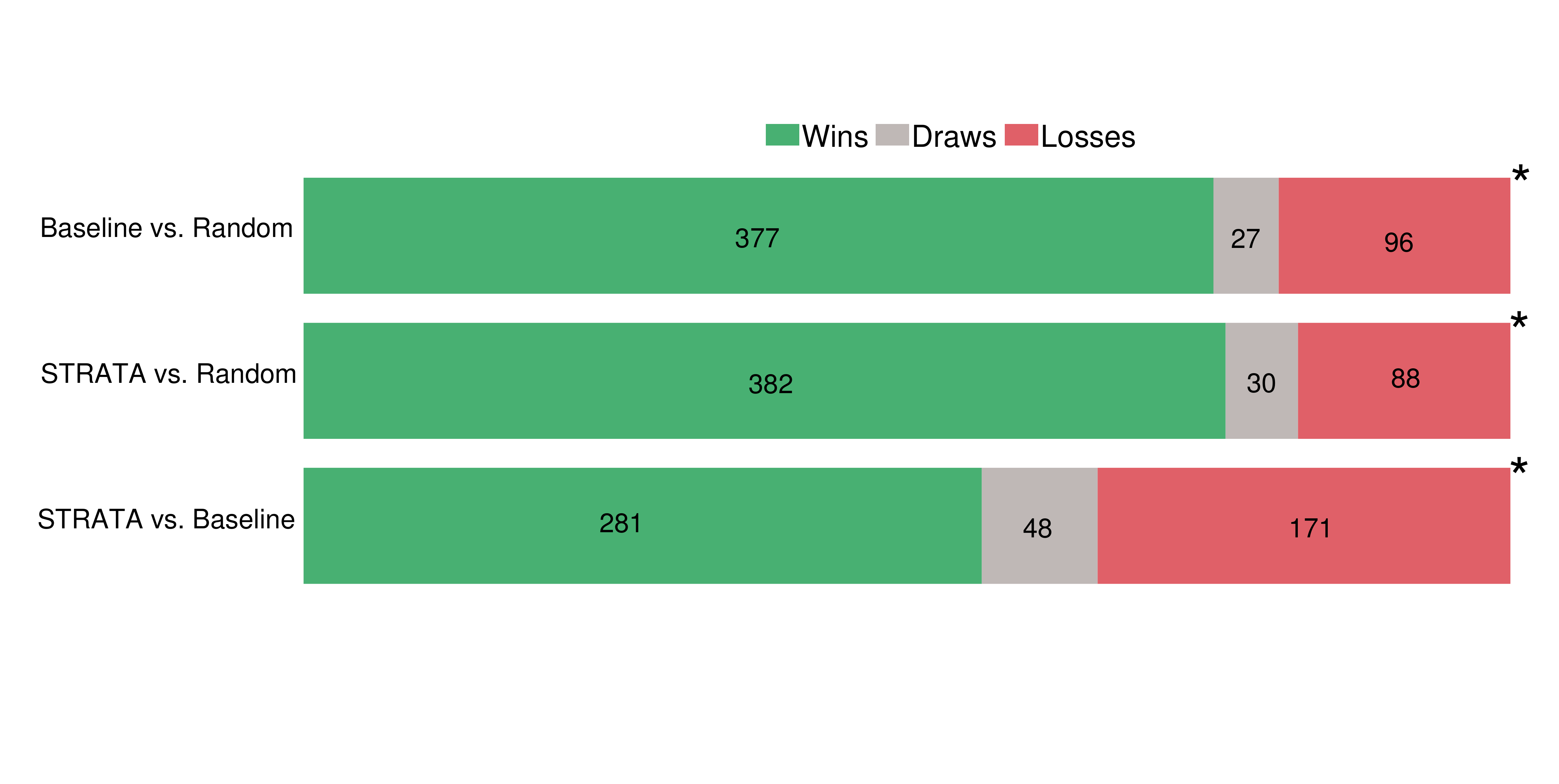}
    \caption{\hl{Relative performances of random task assignment, the baseline framework} \cite{prorok2017impact} \hl{and STRATA across $500$ runs of the capture-the-flag game. An asterisk indicates the statistically significant ($p<0.001$) difference between the proportion of wins to losses in each match-up.}}
    \label{fig:ctf_wins}
\end{figure*}

\textit{\textbf{Discussion}}: \hl{As shown in Fig. {\ref{fig:ctf_wins}}, given appropriate $\bar{Y}$ and $Y^*$, both the baseline framework and STRATA are more likely to win against random task assignment. However, when compared head-to-head with the same $\bar{Y}$ and $Y^*$, STRATA is more likely to win against the baseline framework. Further, based on the $z$-test, we find that the proportions of wins are statistically significantly ($p<0.001$) higher than those of losses in all three conditions. These observations are likely due to the fact that the baseline framework implicitly reasons about the average trait values when corresponding the modified task requirements $\bar{Y}$ that is compatible to the binary trait space. This type of reasoning, while limited, is still more effective than not reasoning about any of the factors that influence team performance. However, similar to our observations in the simulated experiments, reasoning about continuous trait models along with inter- and intra- species variations is not considerably more effective than reasoning about binary approximations of traits. Thus, STRATA's ability to reason about the traits and task requirements translates to high-level team performance.}

\section{Conclusion}
We presented STRATA, a unified framework capable of effective task assignments in large teams of heterogeneous agents. The members of the team are modeled as belonging to different species, each defined by a set of its capabilities. STRATA models capabilities in the continuous space and explicitly takes into account both species-level and agent-level variations. Further, we quantified the diversity of a given team by introducing two separate notions of \emph{minspecies}, each specifying the minimal subset of species necessary to achieve the corresponding goal. Finally, we illustrated the necessity and effectiveness of STRATA using two sets of experiments. The experimental results demonstrate that STRATA (1) successfully distributes a large heterogeneous team to meet complex task requirements, (2) consistently performs better than the baseline framework that only considers binary traits, and (3) results in improved higher-level team performance in a simulated game of capture-the-flag.

\hl{Our proposed approach and its limitations have revealed exciting avenues for future research. Firstly, STRATA can be extended to consider the importance of different tasks or traits, resulting in preferential or importance-based task assignment. Secondly, given the complexity of the optimization problem, it is yet unclear how to provide guarantees on the optimality of the solution. Thirdly, to minimize risk, STRATA reduces variance in expected trait distribution by choosing the most consistent species for tasks. However, it might beneficial in scenarios with very limited resources to consider maximizing trait variance in the interest of the possibility of higher rewards. Finally, we have introduced new measures of diversity based in a comprehensive class of trait models. It would be interesting to explore how to effectively utilize these measures in order study the trade-offs involved in the synthesis and composition of effective multi-agent teams.}

\section*{Acknowledgement}
\hl{We would like to thank the anonymous reviews for their constructive feedback that greatly helped improve the quality of this article.} This work was supported by the Army Research Lab under Grant W911NF-17-2-0181 (DCIST CRA).

%\section*{References}
\bibliography{main}{}
\bibliographystyle{spbasic}
\end{document}